\documentclass[a4paper,12pt]{article}

\usepackage{url}
\usepackage{bm}
\usepackage{balance}

\usepackage{lmodern}
\usepackage[utf8]{inputenc}
\usepackage[T1]{fontenc}

\usepackage[american]{babel}
\usepackage{csquotes}

\usepackage[babel = true]{microtype}
\usepackage{xspace}

\usepackage[dvipsnames]{xcolor}
\definecolor{stroke1}{HTML}{2574A9}

\usepackage{amsmath}
\usepackage{amssymb}
\usepackage{amsthm}
\usepackage{mathtools}
\usepackage{dsfont}

\newtheorem{theorem}{Theorem}
\newtheorem{lemma}{Lemma}
\newtheorem{corollary}{Corollary}

\usepackage[numbers]{natbib}

\makeatletter
    \def\NAT@spacechar{~}
\makeatother

\usepackage
[
	ruled,
	vlined,
	linesnumbered,
]
{algorithm2e}

\usepackage
[
	bookmarks = true,
	bookmarksopen = false,
	bookmarksnumbered = true,
	pdfstartpage = 1,
	colorlinks = true,
	allcolors = stroke1,
]
{hyperref}
\usepackage
[
	noabbrev,
	nameinlink,
]
{cleveref}

\newcommand*{\N}{\mathds{N}}
\newcommand*{\R}{\mathds{R}}
\newcommand*{\OM}{\textsc{OneMax}\xspace}
\newcommand*{\LO}{\textsc{LeadingOnes}\xspace}
\newcommand*{\E}{\mathrm{E}}

\begin{document}

    \title{A Simplified Run Time Analysis of the Univariate Marginal Distribution Algorithm on LeadingOnes}
    
 \author{Benjamin Doerr\\ Laboratoire d'Informatique (LIX)\\ CNRS\\ \'Ecole Polytechnique\\ Institut Polytechnique de Paris\\ Palaiseau\\ France
 \and
 Martin S. Krejca\\ Hasso Platter Institute\\ University of Potsdam\\ Potsdam\\ Germany}

 \maketitle
    
    \begin{abstract}
        With elementary means, we prove a stronger run time guarantee for the univariate marginal distribution algorithm (UMDA) optimizing the \LO benchmark function in the desirable regime with low genetic drift. If the population size is at least quasilinear, then, with high probability, the UMDA samples the optimum within a number of iterations that is linear in the problem size divided by the logarithm of the UMDA's selection rate. This improves over the previous guarantee, obtained by Dang and Lehre (2015) via the deep level-based population method, both in terms of the run time and by demonstrating further run time gains from small selection rates. With similar arguments as in our upper-bound analysis, we also obtain the first lower bound for this problem. Under similar assumptions, we prove that a bound that matches our upper bound up to constant factors holds with high probability.
    \end{abstract}
    

{\sloppy
\section{Introduction}

Estimation-of-distribution algorithms (EDAs) are randomized search-heuristics that create a probabilistic model of the search space and refine it iteratively. In each iteration, the current model of an EDA is used to create some samples which, in turn, are used to adjust the model such that better solutions are more likely to be created in the following iteration. Thus, the model evolves over time into one that creates very good solutions. EDAs have been applied to real-world problems with great success~\cite{PelikanHL15SurveyOnEDAs}.

Within the last few years, the theoretical analysis of EDAs has gained increasing interest, as summarized by Krejca and Witt~\cite{KrejcaW18EDABookChapter}. One of the first papers in this period was by Dang and Lehre~\cite{DangL15UMDAonLO}, who proved run time guarantees for the \emph{univariate marginal distribution algorithm} (UMDA, \cite{MuehlenbeinP96UMDA}) when optimizing the two classical benchmark functions \OM and \LO. While their run time bound for \OM has been improved since then independently by Lehre and Nguyen~\cite{LehreN17UMDAonOneMax} and Witt~\cite{Witt17UMDAonOneMax}, the run time bound of $O(n^2 + n \lambda \log \lambda)$, where~$n$ is the problem dimension and~$\lambda$ is the offspring population size of the UMDA, is the best known result so far on \LO.\footnote{In an extension of~\cite{DangL15UMDAonLO}, Dang~et~al.~\cite{DangLN19UMDAAnalysis} show the same run time bound but slightly improve the required population sizes.}

In this work, we improve with \Cref{thm:UMDA_on_LO} the second term of this bound from $O(n \lambda \log \lambda)$ to $O\big(n \frac{\lambda}{\log(\lambda/\mu)}\big)$ when $\mu = \Omega(n \log n)$, where $\mu \leq \lambda$ is the size of the subpopulation selected for the model update. In the regime of $\mu = \Omega(n \log n)$, the UMDA shows the generally desirable behavior of low genetic drift, that is, the sampling frequencies stay in the middle range of, say, $(\frac{1}{4}, \frac{3}{4})$ until a sufficiently strong fitness signal moves them into the right direction. While EDAs are not necessarily inefficient in the presence of stronger genetic drift, their optimization behavior then often becomes similar to a slowed-down version of the (1+1)~evolutionary algorithm. Genetic drift, however, can also lead to a performance loss, since it may take long to move a frequency from the wrong boundary value back into the middle range. This has been rigorously shown by Lengler~et~al.~\cite{LenglerSW18cGAMediumStepSize}.

Equally interesting to the improved run time guarantee is our elementary proof method. While it was truly surprising that Dang and Lehre~\cite{DangL15UMDAonLO} could use the level-based population method to analyze an EDA (which does not have a population that is transferred from one iteration to the next), this method is a highly advanced tool and one that can be difficult to use. In contrast to this, our proof only uses elementary arguments common in the analysis of evolutionary algorithms. We are thus optimistic that our arguments can more easily be applied to other EDAs as well.

We further demonstrate the usability of our proof method by proving a matching lower bound (see \Cref{thm:UMDA_on_LO_lower_bound}), which improves the previously best known lower bounds by Lehre and Nguyen~\cite{LehreN19UMDALowerBounds} for the regime of $\mu = \Omega(n \log n)$. For the regime of $\mu = \Omega(\log n) \cap o(n \log n)$, the bound $\Omega\left(\frac{n \lambda}{\log(\lambda - \mu)}\right)$ by Lehre and Nguyen remains the best known lower bound. Additionally, Lehre and Nguyen prove a lower bound of $e^{\Omega(\mu)}$ for $\mu = \Omega(\log n)$ and $\lambda \lesssim e \mu$, which remains untouched by our result.

We note that both of our bounds do not require the fraction $\mu/\lambda$ to be constant, which is a common requirement of many other analyses of the UMDA~\cite{DangLN19UMDAAnalysis,LehreN17UMDAonOneMax,Witt17UMDAonOneMax,KrejcaW17UMDAOneMaxLowerBound} (although this is not always explicitly stated in the result). In particular, our bounds show that the gain from reducing the selection rate $\mu/\lambda$ (which often requires a costly increase of~$\lambda$) is very small, namely, only logarithmic in~$\frac{1}{\mu/\lambda}$.

Another advantage of our approach is that it gives run time guarantees that hold with high probability, whereas the level-based method, relying on drift arguments, can only give bounds on expected run times. Consequently, the result of Dang and Lehre~\cite{DangL15UMDAonLO} also concerns the expectation only. We believe that a result that holds with high probability is often more relevant, as has also been argued by Doerr~\cite{Doerr19cGAonJump}. 
%

\section{Preliminaries}
\label{sec:preliminaries}

We are concerned with the run time analysis of algorithms optimizing pseudo-Boolean functions, that is, functions $f\colon \{0, 1\}^n \to \R$, where $n \in \N$ denotes the dimension of the problem. Given a pseudo-Boolean function~$f$ and a bit string~$x$, we refer to~$f$ as a \emph{fitness function,} to~$x$ as an \emph{individual,} and to $f(x)$ as the \emph{fitness of~$x$.}

For an $n \in \N$, we define $[n] = [1, n] \cap \N$. From now on, if not stated otherwise, the variable~$n$ always denotes a natural number. For a vector~$x$ of length~$n$, we denote its component at position $i \in [n]$ via $x_i$. 

We consider the optimization of the pseudo-Boolean function $\LO\colon$ $\{0, 1\}^n \to \{0\} \cup [n]$, which states for a bit string of length~$n$ the longest prefix of leading~$1$s within that bit string. More formally, for all $x \in \{0, 1\}^n$,
\begin{align*}
	\LO(x) = \sum_{i = 1}^{n} \prod_{j = 1}^{i} x_i\ .
\end{align*}
Note that the all-$1$s bit string is the unique global optimum of \LO.

Our algorithm of interest is the UMDA (\Cref{alg:UMDA}) with parameters $\mu, \lambda \in \N$, $\mu \leq \lambda$. It maintains a vector~$p$ of probabilities (the \emph{frequency vector}) of length~$n$, whose components we call \emph{frequencies,} and it updates this vector iteratively in the following way: first, $\lambda$ individuals are created independently from another such that, for each individual $x \in \{0, 1\}^n$ and each position $i \in [n]$, it holds that $x_i$ is~$1$ with probability~$p_i$ and~$0$ otherwise. Then, from these~$\lambda$ individuals, a subset of~$\mu$ individuals with the highest fitness is chosen (breaking ties uniformly at random), and, for each position $i \in [n]$, the frequency~$p_i$ is set to the relative number of~$1$s at position~$i$ among the~$\mu$ best individuals. Last, if a frequency~$p_i$ is below~$\frac{1}{n}$, it is increased to~$\frac{1}{n}$, and if it is above $1 - \frac{1}{n}$, it is decreased to $1 - \frac{1}{n}$. This circumvents frequencies from being stuck at the extremal values~$0$ or~$1$. We denote the frequency vector of iteration $t \in \N$ by $p^{(t)}$. Note that we start with iteration $t = 0$.

\begin{algorithm}
	\caption{\label{alg:UMDA} The UMDA~\cite{MuehlenbeinP96UMDA} with parameters~$\mu$ and~$\lambda$, $\mu \leq \lambda$, maximizing a fitness function $f\colon \{0, 1\}^n \to \R$ with $n \geq 2$}
	
    $t \gets 0$\;
	$p^{(t)} \gets (\tfrac{1}{2})_{i \in [n]}$\;
	\Repeat( $\vartriangleright$ \emph{iteration}~$t$)
	{\emph{termination criterion met}}
	{
		\lFor{$i \in [\lambda]$}
		{
			$x^{(i)} \gets \text{individual sampled via } p^{(t)}$%
		}
		let $y^{(1)}, \ldots, y^{(\mu)}$ denote the~$\mu$ best individuals out of $x^{(1)}, \ldots, x^{(\lambda)}$ (breaking ties uniformly at random)\;
		\lFor{$i \in [n]$}
		{
			$p^{(t + 1)}_i \gets \frac{1}{\mu} \sum_{j = 1}^{\mu} y^{(i)}_j$%
		}
		restrict~$p^{(t + 1)}$ to the interval $[\tfrac{1}{n}, 1 - \tfrac{1}{n}]$\;
	}
\end{algorithm}

In the context of optimizing \LO, we say that a position $i \in [n]$ of~$p^{(t)}$ is \emph{critical} in iteration $t \in \N$ if and only if all of the frequencies at indices less than~$i$ are $1 - \frac{1}{n}$ and~$p^{(t)}_i$ is less than $1 - \frac{1}{n}$. Intuitively, a critical frequency is the next one that needs to be set to $1 - \frac{1}{n}$ in order to optimize \LO efficiently.

When analyzing the run time of the UMDA optimizing a fitness function~$f$, we are interested in the number~$T$ of fitness function evaluations until an optimum of~$f$ is sampled for the first time. Since the UMDA is a randomized algorithm, $T$ is a random variable, and we are interested in a bound on~$T$ that holds with high probability. Note that the run time~$T$ of the UMDA is at most~$\lambda$ times the number~$I$ of iterations until an optimum is sampled for the first time. Likewise, $T$ is at least $(I - 1)\lambda + 1$.

In order to prove statements on random variables that hold with high probability, we use the following commonly known Chernoff bounds.
\begin{theorem}[Chernoff bound~{\cite[Theorem~$1.10.5$]{Doerr20bookchapter}}]
	\label{thm:chernoff_smaller}
	Let $k \in \N$, $\delta \in [0, 1]$, and let~$X$ be the sum of~$k$ independent random variables, each taking values in $[0, 1]$. Then
	\begin{align*}
		\Pr\!\big[X \leq (1 - \delta) \E[X]\big] \leq e^{- \frac{\delta^2 \E[X]}{2}}\ .
	\end{align*}
\end{theorem}

\begin{theorem}[Chernoff bound~{\cite[Theorem~$1.10.1$]{Doerr20bookchapter}}]
	\label{thm:chernoff_larger}
	Let $k \in \N$, $\delta \in [0, 1]$, and let~$X$ be the sum of~$k$ independent random variables, each taking values in $[0, 1]$. Then
	\begin{align*}
		\Pr\!\big[X \geq (1 + \delta) \E[X]\big] \leq e^{- \frac{\delta^2 \E[X]}{3}}\ .
	\end{align*}
\end{theorem}

The next two theorems, recently proven in~\cite{DoerrZ19}, give upper bounds on the negative effect of genetic drift on the UMDA. The first result considers the optimization of fitness functions~$f$ that \emph{weakly prefer} a~$1$ at a position $i \in [n]$, that is, for all bit strings $x, x' \in \{0, 1\}^n$ with $x_i = 1$, $x'_i = 0$, and $x_j = x'_j$ for all other positions $j \in [n] \setminus \{i\}$, it holds that $f(x) \geq f(x')$. In other words, having a~$1$ at position~$i$ always yields a fitness at least as good as when having a~$0$ at~$i$. Note that \LO weakly prefers a $1$ in all bit positions. The theorem states that the frequency at such a position~$i$ does not drop far below its initial value~$\frac{1}{2}$ for a long time.

\begin{theorem}[{\cite[Theorem~$7$]{DoerrZ19}}]
    \label{thm:neutralFrequency}
    Consider the UMDA with parameters~$\mu$ and~$\lambda$ optimizing a function~$f$ that weakly prefers a~$1$ at position $i \in [n]$. Then, for all $d > 0$ and all iterations $t \in \N$, we have
    \begin{align*}
        \Pr\!\big[\forall t' \in [0 .. t]\colon p_i^{(t')} > \tfrac{1}{2} - d\big] \geq 1 - 2 e^{- \frac{d^2 \mu}{2 t}}\ .
    \end{align*}
\end{theorem}

The next theorem considers the case that there is no preference for a bit value at position $i \in [n]$, that is, for all bit strings $x, x' \in \{0, 1\}^n$ with $x_i = 1$, $x'_i = 0$, and $x_j = x'_j$ for all other positions $j \in [n] \setminus \{i\}$, it holds that $f(x) = f(x')$. Given this assumption, we call position~$i$ \emph{neutral.}

\begin{theorem}[{\cite[Corollary~$2$]{DoerrZ19}}]
    \label{thm:neutralFrequencyBothDirections}
    Consider the UMDA with parameters~$\mu$ and~$\lambda$ optimizing a function~$f$ such that position $i \in [n]$ is neutral. Then, for all $d > 0$ and all iterations $t \in \N$, we have
    \begin{align*}
        \Pr\!\big[\forall t' \in [0 .. t]\colon p_i^{(t')} \in (\tfrac{1}{2} - d, \tfrac{1}{2} + d)\big] \geq 1 - 2 e^{- \frac{d^2 \mu}{2 t}}\ .
    \end{align*}
\end{theorem}

\section{Upper Bound}
\label{sec:runTimeResult}

In the following, we present our simple and intuitive run time analysis for the upper bound of the UMDA optimizing \LO, which gives the following theorem.

\begin{theorem}
	\label{thm:UMDA_on_LO}
	Let $\delta \in (0, 1)$ be a constant, and let $\zeta = \frac{1 - \delta}{4e}$. Consider the UMDA optimizing \LO with $\mu \geq 128 n \ln n$ and $\lambda \geq \frac{\mu}{\zeta}$. Further, let $d = \lfloor\log_4(\zeta \frac{\lambda}{\mu})\rfloor$. Then the UMDA samples the optimum after at most $\lambda (\lceil\frac{n}{d + 1}\rceil + \lceil \frac{n}{n - 1} e \ln n\rceil)$ fitness function evaluations with a probability of at least $1 - 5 n^{-1}$.
\end{theorem}

As discussed in the introduction, we only want to consider the regime with low genetic drift. Hence, we first argue that no frequency drops below~$\frac{1}{4}$ before the optimum is sampled (\Cref{lem:frequencies_do_not_drop_too_low}). Then we show that, in this case, in each iteration, roughly $\log \frac{\lambda}{\mu}$ additional frequencies are set to $1 - \frac{1}{n}$. More specifically, if $i \in [n]$ is critical, then all frequencies at positions roughly up to $i + \log \frac{\lambda}{\mu}$ are set to $1 - \frac{1}{n}$ (\Cref{lem:increasing_a_frequency}). Thus, a total of roughly $\frac{n}{\log(\lambda/\mu)}$ iterations suffice to move all frequencies to $1 - \frac{1}{n}$. From such a state, the optimum is sampled with high probability after a logarithmic number of iterations.

We start by proving that the following parameter setting ensures that no frequency drops below the value $\frac{1}{4}$ within $2 n$ iterations with high probability.

\begin{lemma}
	\label{lem:frequencies_do_not_drop_too_low}
	Consider the UMDA with $\lambda \geq \mu \geq 128 n \ln n$. Assume that it optimizes a function that weakly prefers a~$1$ at all positions. Then, with a probability of at least $1 - 2 n^{-1}$, each frequency will stay at a value of at least~$\frac{1}{4}$ for the first $2 n$ iterations.
\end{lemma}

\begin{proof}
	Consider an iteration $t \leq 2n$ as well as a position $i \in [n]$. By \Cref{thm:neutralFrequency} with $d = \frac{1}{4}$, we see that the probability that~$p_i$ drops below~$\frac{1}{4}$ within the first $t \leq 2n$ iterations is at most $2 e^{-\mu / (32 \cdot t)} \leq 2 e^{-\mu / (32 \cdot 2n)} \leq 2 n^{-2}$, where we used our bound on~$\mu$. Applying a union bound over all~$n$ frequencies gives the claim.
\end{proof}

We now prove that a critical frequency, all its preceding frequencies, as well as roughly $\log \frac{\lambda}{\mu}$ following frequencies are set to $1 - \frac{1}{n}$ within a single iteration. That is, we increase roughly $1 + \log \frac{\lambda}{\mu}$ new frequencies to their maximum value.

\begin{lemma}
	\label{lem:increasing_a_frequency}
	Let $\delta \in (0, 1)$ be a constant, and let $\zeta = \frac{1 - \delta}{4e}$. Consider the UMDA optimizing \LO with $\mu \geq 4 \frac{1 - \delta}{\delta^2} \ln n$ and $\lambda \geq \frac{\mu}{\zeta}$. Furthermore, consider an iteration $t \in \N$ such that position $i \in [n]$ is critical and that, for all positions $j \geq i$, we have $p_j^{(t)} \geq \frac{1}{4}$. Let $d = \lfloor\log_4(\zeta \frac{\lambda}{\mu})\rfloor$. Then, with a probability of at least $1 - n^{-2}$, for all positions $j \in [\min\{n, i + d\}]$, we have $p_j^{(t + 1)} = 1 - \frac{1}{n}$.
\end{lemma}

\begin{proof}
	Note that $d \geq 0$ due to our assumption on~$\lambda$. We look at the population of~$\lambda$ individuals that is sampled in iteration~$t$ and determine the number~$X$ of individuals that have at least $i' \coloneqq \min\{n, i + d\}$ leading $1$s. Since the frequencies at all positions less than~$i$ are at $1 - \frac{1}{n}$, the probability that all of these frequencies sample a~$1$ for a single individual is $(1 - \frac{1}{n})^{i - 1} \geq (1 - \frac{1}{n})^{n - 1} \geq \frac{1}{e}$. Further, since the probability to sample a~$1$ at positions at least~$i$ is at least $\frac{1}{4}$, we have $\E[X] \geq \frac{\lambda}{e} \cdot 4^{-(1 + d)} \geq \frac{\mu}{4e \zeta} \geq \frac{\mu}{1 - \delta}$.
	
	We now apply \Cref{thm:chernoff_smaller} in order to show that it is unlikely that fewer than~$\mu$ individuals from iteration~$t$ have fewer than~$i'$ leading~$1$s. Using our bounds on~$\mu$ and our estimate on $\E[X]$ from above, we compute
	\begin{align*}
		\Pr[X < \mu] &\leq \Pr\big[X \leq (1 - \delta) \E[X]\big] \leq e^{-\frac{\delta^2 \E[X]}{2}}\\
        &\leq e^{-\frac{\delta^2}{2(1 - \delta)} \mu} \leq n^{-2} .
	\end{align*}
	Thus, with a probability of at least $1 - n^{-2}$, at least~$\mu$ individuals have at least~$i'$ leading~$1$s.
	
	Since the UMDA is optimizing \LO, in this case, all of the selected top~$\mu$ individuals have at least~$i'$ leading~$1$s, which results in all frequencies at positions in $[i']$ being set to $1 - \frac{1}{n}$, that is, for all $j \in [i']$, we have $p_j^{(t + 1)} = 1 - \frac{1}{n}$.
\end{proof}

We now prove our main result.

\begin{proof}[Proof of \Cref{thm:UMDA_on_LO}]
	We prove that the UMDA samples the optimum after $\lceil\frac{n}{d + 1}\rceil + \lceil \frac{n}{n - 1} e \ln n\rceil$ iterations with a probability of at least $1 - 5 n^{-1}$. Since it performs~$\lambda$ fitness function evaluations each iteration, the theorem follows.
	
	Since \LO weakly prefers~$1$ at all positions, by \Cref{lem:frequencies_do_not_drop_too_low} and $\mu \geq 128 n \lceil\ln n\rceil$, no frequency drops below~$\frac{1}{4}$ within $2n$ iterations with a probability of at least $1 - 2 n^{-1}$.
	
	Consider an iteration $t \leq n$ such that position $i \in [n]$ is critical. Note that $\mu \geq 4 \lceil\frac{1 - \delta}{\delta^2} \ln n\rceil$ for sufficiently large~$n$. By \Cref{lem:increasing_a_frequency}, with a probability of at least $1 - n^{-2}$, for each frequency at position in $j \in [\min\{n, i + d\}]$, we have $p_j^{(t + 1)} = 1 - \frac{1}{n}$. That is, $d + 1$ additional frequencies are set to $1 - \frac{1}{n}$. Applying a union bound for the first $2 n$ iterations of the UMDA shows that all frequencies are at $1 - \frac{1}{n}$ after the first $\lceil\frac{n}{d + 1}\rceil$ iterations and stay there for at least $n$ additional iterations with a probability of at least $1 - 2 n^{-1}$.
	
	Consequently, after the first~$n$ iterations, the optimum is sampled in each iteration with a probability of $(1 - \frac{1}{n})^n \geq (1 - \frac{1}{n})\frac{1}{e}$. Thus, after $\lceil \frac{n}{n - 1} e \ln n\rceil$ additional iterations, the optimum is sampled with a probability of at least $1 - \big(1 - \frac{n - 1}{n}\frac{1}{e}\big)^{\lceil \frac{n}{n - 1} e \ln n\rceil} \geq 1 - n^{-1}$.
	
	Overall, by applying a union bound over all failure probabilities, the UMDA needs at most $\lceil\frac{n}{d + 1}\rceil + \lceil \frac{n}{n - 1} e \ln n\rceil$ iterations to sample the optimum for the first time with a probability of at least $1 - 5 n^{-1}$.
\end{proof}

We note that we stated explicit constants in the result above as we felt that this eases reading, but we did not try to optimize them. For example, a selection rate of at most some constant less than $\frac{1}{2e}$ can give the same run time guarantee when raising $\lambda$ by a sufficiently large constant factor. A selection rate of at most some constant less than $\frac{1}{e}$ can also be tolerated. Now it takes a constant number of iterations to move a critical frequency to $1 - \frac{1}{n}$, so the run time guarantee increases by a constant factor.

\section{Lower Bound}
\label{sec:lowerBound}

Our main insight, which gave our sharper upper bound with a proof simpler than in previous works, was that the UMDA, when optimizing \LO in the regime of low genetic drift, makes a steady progress in each iteration: It sets the frequencies to the maximum value $1 - \frac 1n$ in a left-to-right fashion, keeping the other frequencies close to the middle value of $\frac 12$. The increase of the number of frequencies at the maximum value, with a simple Chernoff bound argument, could be shown to be  logarithmic in the reciprocal $\frac{1}{\mu/\lambda}$ of the selection rate.

In this section, we show that the same proof approach (with small modifications) can also be employed to show lower bounds, and in this case, a matching lower bound, which also is the first lower bound for this setting at all. 
%
%

\begin{theorem}
    \label{thm:UMDA_on_LO_lower_bound}
    Let $\delta \in (0, 1)$ be a constant, and let $\zeta = \frac{3}{4}(1 + \delta)$. Consider the UMDA optimizing \LO with $\lambda \geq \mu \geq 64 n \ln n$ and $\lambda \geq \frac{\mu}{\zeta}$. Further, let $d = \lceil\log_{4/3}(\zeta \frac{\lambda}{\mu})\rceil$, and let $\xi = \lceil\log_{4/3} (n^2\lambda)\rceil + 1$. Then the UMDA samples the optimum after more than $\lambda\lfloor\frac{n - \xi}{d + 1}\rfloor$ fitness function evaluations with a probability of at least $1 - 4n^{-1}$.
\end{theorem}

To prove a lower bound via the general idea laid out above, we need to show that frequencies that do not receive a fitness signal do not approach $1-\frac 1n$ due to genetic drift. Here we have to be slightly more careful than in our upper bound analysis, since now the fitness signal does move the frequencies into the undesired (from the view-point of lower bound proofs) direction. Consequently, we can employ the low-genetic drift argument only while we are sure that we do not receive a fitness signal (\Cref{lem:frequencies_are_bounded}). 

Using a Chernoff-type concentration argument (which in principle works similarly for upper and lower bounds), we show that at most roughly $\log \frac{\lambda}{\mu}$ frequencies above the critical position receive a fitness signal (and thus potentially leave the middle range), see~\Cref{lem:increasing_not_too_many_frequencies}. 

Consequently, in the first $O(\frac{n}{\log \lambda / \mu})$ iterations, we have many frequencies that are far from the maximum value, and thus sampling the optimum is unlikely (\Cref{lem:sampling_optimum_is_unlikely}). This yields our lower bound.

To make these arguments precise, we define when a frequency of the UMDA stops being neutral, that is, receives a fitness signal. To this end, we say that a position $i \in [n]$ is \emph{selection-relevant} (with respect to \LO) in iteration $t \in \N$ if and only if the offspring population of the UMDA in iteration~$t$ has at least~$\mu$ individuals with at least $i - 1$ leading~$1$s. Thus, with respect to selection, the bit value at position~$i$ decides whether an individual is selected for the update or not. We call the largest selection-relevant position in an iteration the \emph{maximum} selection-relevant position. Note that all positions greater than the maximum selection-relevant position are neutral during this iteration.

Since, by the definition of a selection-relevant position~$i$, all frequencies at positions less than~$i$ are set to $1 - \frac{1}{n}$, the critical position for the next iteration is~$i$, too. Thus, bounding the progress of the selection-relevant position also bounds the overall progress of the UMDA on \LO.

We start by showing that each frequency stays in the interval $(\frac{1}{4}, \frac{3}{4})$ until its position becomes selection-relevant.

\newcommand*{\selectionRelevantTime}{t^{\mathrm{sel}}}
\begin{lemma}
	\label{lem:frequencies_are_bounded}
	Consider the UMDA with $\lambda \geq \mu \geq 64 n \ln n$. Further, for each position $i \in [n]$, let $t'_i \in \N$ denote the first iteration such that position~$i$ is selection-relevant, and let $\selectionRelevantTime_i = \min \{t'_i, n\}$. Then, with a probability of at least $1 - 2 n^{-1}$, within the first~$n$ iterations, for each position $i \in [n]$ and for each iteration $t \leq \selectionRelevantTime_i$, it holds that $p_i^{(t)} \in (\frac{1}{4}, \frac{3}{4})$.
\end{lemma}

\begin{proof}
    Consider a position $i \in [n]$. Note that, for all iterations $t \leq \selectionRelevantTime_i$, the frequency $p_i$ is neutral. By \Cref{thm:neutralFrequencyBothDirections} with $d = \frac{1}{4}$, we see that the probability that~$p_i$ leaves the interval $ (\frac{1}{4}, \frac{3}{4})$ within the first $\selectionRelevantTime_i \leq n$ iterations is at most $2 e^{-\mu / (32 \cdot \selectionRelevantTime_i)} \leq 2 e^{-\mu / (32 \cdot n)} \leq 2 n^{-2}$, where we used our lower bound on~$\mu$.
    
    Applying a union bound over all~$n$ frequencies yields that at least one frequency leaves the interval $(\frac{1}{4}, \frac{3}{4})$ within the first~$n$ iterations before being selection-relevant with a probability of at most $2 n^{-1}$, which concludes the proof.
\end{proof}

We now show that the maximum selection-relevant position is only roughly $\log \frac{\lambda}{\mu}$ larger than the critical position during each iteration.

\begin{lemma}
    \label{lem:increasing_not_too_many_frequencies}
    Let $\delta \in (0, 1)$ be a constant, and let $\zeta = \frac{3}{4}(1 + \delta)$. Consider the UMDA optimizing \LO with $\mu \geq 6 \frac{1 + \delta}{\delta^2} \ln n$ and $\lambda \geq \mu \cdot \max\{1, \frac{1}{\zeta}\}$. Furthermore, consider an iteration $t \in \N$ such that position $i \in [n]$ is critical and that, for all positions $j > i$, we have $p_j^{(t)} \leq \frac{3}{4}$. Let $d = \lceil\log_{4/3}(\zeta \frac{\lambda}{\mu})\rceil$. Then, with a probability of at least $1 - n^{-2}$, the maximum selection-relevant position for iteration~$t$ is at most $\min\{n, i + d + 1\}$.
\end{lemma}

\begin{proof}
    Note that $d \geq 0$ by our assumption on~$\lambda$. Similar to the proof of \Cref{lem:increasing_a_frequency}, we consider the offspring population of~$\lambda$ individuals sampled in iteration~$t$. Let~$X$ denote the number of individuals that have at least $i' \coloneqq \min\{n, i + d + 1\}$ leading~$1$s. By assumption, all frequencies at positions greater than~$i$ are at most~$\frac{3}{4}$. Thus, $\E[X] \leq \lambda (\frac{3}{4})^{1 + d} = \lambda (\frac{4}{3})^{-(1 + d)} \leq \frac{3}{4}\frac{\mu}{\zeta} \leq \frac{\mu}{1 + \delta}$.
    
    We now apply \Cref{thm:chernoff_larger} in order to show that it is unlikely that at least~$\mu$ individuals from iteration~$t$ have at least~$i'$ leading~$1$s. Using our bounds on~$\mu$ and our estimate on $\E[X]$ from above, we compute
    \begin{align*}
        \Pr[X \geq \mu] &\leq \Pr\big[X \geq (1 + \delta)\E[X]\big] \leq e^{-\frac{\delta^2 \E[X]}{3}}\\
        &\leq e^{-\frac{\delta^2 \mu}{3(1 + \delta)}} \leq n^{-2}.
    \end{align*}
    Thus, with a probability of at least $1 - n^{-2}$, fewer than~$\mu$ individuals have at least~$i'$ leading~$1$s. This means that the maximum selection-relevant position in this iteration is in $[i']$.
\end{proof}


Before we prove our lower bound, we show that the UMDA does not sample the optimal solution of \LO with high probability while the critical position is at least logarithmically far away from the end.

\begin{lemma}
    \label{lem:sampling_optimum_is_unlikely}
    Consider the UMDA optimizing \LO with $\lambda \geq \mu$. Further, consider an iteration $t \in \N$ and a position $i \in [n]$ such that, for all positions $j > i$, we have $p_j^{(t)} \leq \frac{3}{4}$. Then, with a probability of at least $1 - \lambda (\frac{3}{4})^{n - i}$, the UMDA does not sample the optimum in this iteration.
\end{lemma}

\begin{proof}
    By our assumption on the frequencies and on~$i$, the probability that a single individual in the offspring population in iteration~$t$ is the all-$1$s string (that is, the optimum of \LO) is at most $(\frac{3}{4})^{n - i}$. Thus, the probability that none of the~$\lambda$ offspring is optimal is, by Bernoulli's inequality, at least $\big(1 - (\frac{3}{4})^{n - i}\big)^{\lambda} \geq 1 - \lambda (\frac{3}{4})^{n - i}$, as desired.
\end{proof}

We now prove our lower bound.

\begin{proof}[Proof of \Cref{thm:UMDA_on_LO_lower_bound}]
    We prove that the UMDA needs, with a probability of at least $1 - 4n^{-1}$, more than $\lfloor\frac{n - \xi}{d + 1}\rfloor$ iterations until it samples the optimum. Since it performs~$\lambda$ fitness function evaluations each iteration, the theorem then follows.
    
    In the following, we assume that all frequencies remain in the interval $(\frac{1}{4}, \frac{3}{4})$ for the first~$n$ iterations as long as they have never been selection-relevant. By \Cref{lem:frequencies_are_bounded}, this happens with a probability of at least $1 - 2n^{-1}$.
    
    We now prove by induction on the iteration index $t \in \N$ that, with a probability of at least $1 - (t + 1) n^{-2}$, for each position $i > 1 + (t + 1)(d + 1)$, we have that position~$i$ is not selection-relevant up to iteration~$t$.
    
    For the base case $t = 0$, note that position $i = 1$ is critical and that all frequencies are~$\frac{1}{2}$ and thus at most $\frac{3}{4}$. By \Cref{lem:increasing_not_too_many_frequencies}, with a probability of at least $1 - n^{-2}$, the maximum selection-relevant position is at most $d + 2$. Thus, all positions greater than $d + 2$ are not selection-relevant up to iteration~$0$.
    
    For the inductive step, we assume that our inductive hypothesis holds up to iteration~$t$. Note that this means that, with a probability of at least $1 - (t + 1) n^{-2}$, the maximum selection-relevant index is in $[1 + (t + 1)(d + 1)]$ and, thus, the critical position for iteration $t + 1$ is in $[1 + (t + 1)(d + 1)]$. By \Cref{lem:frequencies_are_bounded}, all frequencies at positions greater than $1 + (t + 1)(d + 1)$ are thus at most~$\frac{3}{4}$.\footnote{Note that such frequencies are at most~$\frac{3}{4}$ with a probability of~$1$, as we condition on this event throughout the proof, as stated at the beginning of the proof.} Then, in iteration $t + 1$, again by \Cref{lem:increasing_not_too_many_frequencies}, with a probability of at least $1 - n^{-2}$, the maximum-selection relevant index is at most $1 + (t + 1)(d + 1) + d + 1 = 1 + (t + 2)(d + 1)$. Consequently, via a union bound on the error probabilities of the inductive hypothesis and the current iteration $t + 1$, with a probability of at least $1 - (t + 2) n^{-2}$, no position greater than $1 + (t + 2)(d + 1)$ is selection-relevant up to iteration $t + 1$, which proves our claim.
    
    We now assume that $n - \xi \geq 1$, as \Cref{thm:UMDA_on_LO_lower_bound} is trivial otherwise. Our claim above then yields that, up to iteration $t' \coloneqq \lfloor\frac{n - \xi}{d + 1}\rfloor - 1$, with a probability of at least $1 - \frac{n - \xi}{d + 1} n^{-2} \geq 1 - n^{-1}$, each position greater than $1 + n - \xi$ was never selection-relevant. This means that by \Cref{lem:frequencies_are_bounded}, all such frequencies are at most~$\frac{3}{4}$. By \Cref{lem:sampling_optimum_is_unlikely} with $i = 1 + n - \xi$, with a probability of at least $1 - \lambda (\frac{3}{4})^{n - i} = 1 - \lambda (\frac{3}{4})^{\xi - 1} = 1 - n^{-2}$, the UMDA does not sample the optimum of \LO within a single iteration. Applying a union bound over the first $t' + 1 \leq n$ iterations, with a probability of at least $1 - n^{-1}$, the UMDA does not sample the optimum up to iteration~$t'$ (which are $t' + 1$ iterations).
    
    Overall, by a union bound over all error probabilities, with a probability of at least $1 - 4n^{-1}$, the UMDA does not sample the optimum within the first $t' + 1 = \lfloor \frac{n - \xi}{d + 1} \rfloor$ iterations, which concludes the proof.
\end{proof}

For the sake of completeness, we state the combined result of our upper and lower bound.

\begin{corollary}[combining \Cref{thm:UMDA_on_LO,thm:UMDA_on_LO_lower_bound}]
    \label{cor:tight_run_time_bound}
    Let~$C$ be a sufficiently large constant. Consider the UMDA optimizing \LO with $\lambda \geq C \mu \geq 128 n \ln n$ and with $\lambda$ being bounded from above by a polynomial in~$n$. With a probability of at least $1 - 9n^{-1}$, it samples the optimum after $\Theta\Big(\lambda \frac{n}{\log (\lambda/\mu)}\Big)$ fitness function evaluations.
\end{corollary}

\begin{proof}
    From the assumptions of \Cref{thm:UMDA_on_LO,thm:UMDA_on_LO_lower_bound}, we take the stricter ones. The additive term $\lceil\frac{n}{n - 1} e\ln n\rceil$ in \Cref{thm:UMDA_on_LO} vanishes in asymptotic notation, and the term $n - \xi$ in \Cref{thm:UMDA_on_LO_lower_bound} is $\Omega(n)$, due to~$\lambda$ being bounded from above by a polynomial in~$\xi$. Taking the union bound over the failure probabilities of both theorems concludes the proof.
\end{proof}

\section{Conclusion}
\label{sec:conclusion}

We improved the best known upper bound for the run time of the UMDA on \LO for the case of $\mu \in \Omega(n \log n)$ from $O(n \lambda \log \lambda)$ to $O\big(n \frac{\lambda}{\log(\lambda/\mu)}\big)$. This result improves the previous best result both by removing an unnecessary $\log \lambda$ factor and, not discussed in previous works, by gaining a $\log(\lambda/\mu)$ factor and thus showing an advantage of using a low selection rate $\mu/\lambda$. We obtained these results via a different proof method that avoids the technically demanding level-based method. Our arguments can also be employed for lower bounds. We did so and provided the first lower bound for the run time of the UMDA on \LO. Combined, these results provide a run time estimate for the UMDA on \LO that is tight up to constant factors.

We note that the general proof idea can be extended also to the parameter regime of $\mu \in o(n \log n)$ for the UMDA. We conjecture that a more general upper bound of the UMDA (with $\lambda \in \Omega(\log n)$) on \LO is
\begin{align*}
	O\Bigg(\lambda\left(n + \frac{n}{e^{\mu/n}}\left(\frac{n}{\lambda} + \log \min\{\mu, n\}\right)\right)\Bigg)\ .
\end{align*}
Speaking in terms of iterations and thus ignoring the factor of~$\lambda$, this expression can be explained as follows: the first term of~$n$ considers $O(n)$ frequencies that do not drop below constant values. Each of these frequencies is set to $1 - \frac{1}{n}$ within a constant number of iterations with high probability. Since $\lambda \in \Omega(\log n)$, frequencies at $1 - \frac{1}{n}$ do not drop until the optimum is sampled with high probability.

The second, more complicated term deals with frequencies that, pessimistically, reached the lower border $\frac{1}{n}$. There are $O(n/e^{\mu/n})$ of these frequencies, by the same argument as used in the proof of \Cref{lem:frequencies_do_not_drop_too_low}. The other factor is concerned with the time it takes a critical frequency to be increased to $1 - \frac{1}{n}$ with high probability. Here, a case distinction needs to be made with respect to whether $\mu \geq n$. The inverse of the maximum of~$\mu$ and~$n$ determines the step size in which a critical frequency can be increased. The term $\log \min\{\mu, n\}$ stems from the multiplicative up-drift~\cite{DoerrK19MultiplicativeUpDrift} of such a frequency in order to reach a constant value. Afterward, it is set to $1 - \frac{1}{n}$ within a constant number of iterations (as the first $O(n)$ frequencies). Last, the term $n/\lambda$ is only important if $\lambda \in o(n)$ and denotes the waiting time for a critical frequency to sample at least one~$1$ with~$\lambda$ tries (given that the prefix consists of only~$1$s).

Last, we are positive that our proof technique is applicable to a greater class of EDAs. In order to transfer the proof of the upper bound to other univariate EDAs, only \Cref{lem:frequencies_do_not_drop_too_low,lem:increasing_a_frequency} need to be adjusted to the specific algorithm, which should work similarly for other EDAs too. For the lower bound, \Cref{lem:frequencies_are_bounded,lem:increasing_not_too_many_frequencies,lem:sampling_optimum_is_unlikely} need to be changed.

\section*{Acknowledgments}
This work was supported by a public grant as part of the Investissement d'avenir project, reference ANR-11-LABX-0056-LMH, LabEx LMH, in a joint call with Gaspard Monge Program for optimization, operations research and their interactions with data sciences. This publication is based upon work from COST Action CA15140, supported by COST.


}

\end{document}